\documentclass[10pt,letterpaper]{article}

\usepackage[top=1in,bottom=1in,left=1in,right=1in]{geometry}

\usepackage[utf8]{inputenc} % allow utf-8 input
\usepackage[T1]{fontenc}    % use 8-bit T1 fonts
\usepackage{hyperref}       % hyperlinks
\usepackage{url}            % simple URL typesetting
\usepackage{booktabs}       % professional-quality tables
\usepackage{amsfonts}       % blackboard math symbols
\usepackage{nicefrac}       % compact symbols for 1/2, etc.
\usepackage{microtype}      % microtypography
\usepackage{xcolor}         % colors

\usepackage{amsmath}
\usepackage{amsthm}
\usepackage{amssymb}
\usepackage{color}
\usepackage{graphicx}
\usepackage{bm}
\usepackage{hyperref}
\usepackage{algorithm}
\usepackage{algorithmic}
\usepackage{tikz}
\usepackage{tkz-graph}
\usepackage{float} % to put two algorithms side by side

\newtheorem{theorem}{Theorem}
\newtheorem{lemma}{Lemma}

\newcommand{\maximize}{{\rm maximize\ }}
\newcommand{\minimize}{{\rm minimize\ }}
\newcommand{\subjectto}{{\rm subject\ to\ }}
\newcommand{\R}{\mathbb{R}}

\newcommand{\innerprod}[2]{{\left\langle #1, #2 \right\rangle}}
\newcommand{\tr}{{\rm tr}}
\newcommand{\inv}[1]{{#1}^{-1}}
\newcommand{\diag}{{\rm diag}}
\newcommand{\Diag}{{\rm Diag}}
\newcommand{\zero}{\mathbf{0}}
\newcommand{\one}{\mathbf{1}}
\newcommand{\I}{\mathbf{I}}
\newcommand{\W}{\mathbf{W}}
\newcommand{\X}{\mathbf{X}}
\newcommand{\Fh}{\widehat{\mathbf{F}}}
\renewcommand{\H}{\mathbf{H}}
\newcommand{\M}{\mathbf{M}}
\newcommand{\x}{\mathbf{x}}
\newcommand{\bnu}{\bm{\nu}}
\newcommand{\fh}{\widehat{f}}
\newcommand{\fb}{\overline{f}}
\newcommand{\gh}{\widehat{g}}
\renewcommand{\v}{\mathbf{v}}
\newcommand{\y}{\mathbf{y}}
\newcommand{\z}{\mathbf{z}}
\newcommand{\solve}{{\rm solve}}
\renewcommand{\d}{\mathbf{d}}
\newcommand{\Xopt}{\X^{\rm opt}}
\newcommand{\e}{\mathbf{e}}
\newcommand{\At}{A^*}
\newcommand{\xt}{\x^*}

\title{\textbf{Revisiting the Objective of Echo Chamber Detection}}

\author{Abylaikhan Bexeit\\
CIS, The University of Melbourne\\
\texttt{abylaikhan.bexeit@student.unimelb.edu.au}\\
\and
Kushani Perera\\
CIS, The University of Melbourne\\
\texttt{kushani.perera@unimelb.edu.au}\\
\and
Shanika Karunasekera\\
CIS, The University of Melbourne\\
\texttt{karus@unimelb.edu.au}\\
\and
Jean Honorio\\
CIS, The University of Melbourne\\
\texttt{jean.honorio@unimelb.edu.au}}

\date{}

\begin{document}

\maketitle

\begin{abstract}
In this paper, we study the detection of an echo chamber in a social network, i.e., the identification of a set of nodes that agree on a topic, while disagreeing with the rest of nodes.
We argue that this problem is different from other social network analysis problems such as community detection, and from other graph problems such as maximum graph cut and maximum clique.
To the best of our knowledge, we are the first to formalize the objective function of echo chamber detection, by using the theory of Fourier transforms of set functions~\cite{Stobbe2012}.
We propose scalable semidefinite relaxation, solved via an interior point method and sparse linear algebra.
Experimentally, our algorithm recovers the ground truth echo chamber better than competing methods on small synthetic experiments.
Our algorithm produces echo chambers with better network properties than competing methods on large real-world datasets.
To independently validate our proposed objective function, we show that our algorithm finds echo chambers with more agreements with suspended users than competing methods on a small real-world dataset.
\end{abstract}

\section{Introduction} \label{sec:intro}

Social networks are powerful instruments capable of influencing public opinion, both positively and negatively.
Therefore, social network analysis with the intention of detecting their potential threats to the society is vital.
Echo chambers are identified as a major cause for the widespread of misinformation, aggressive content and extremist ideologies in social networks.
For instance, social media echo chambers have significantly contributed to the promotion of anti-vaccine and flat earth~\cite{cossard2020falling, alatawi2021survey} ideologies, resulting in serious negative repercussions to the society.
Consequently, echo chamber detection plays an important role in social network analysis.

An echo chamber can be defined as a network of users with the same opinion regarding a given topic whose users frequently reinforce the content that supports their pre-existing opinions while discrediting and excluding dissenting opinions~\cite{alatawi2021survey}.
To detect echo chambers, existing work maps echo chamber detection to a community detection problem and apply community detection algorithms~\cite{cossard2020falling, del2017mapping, cinelli2021echo, perera2024quantifying}.
However, direct application of existing community detection algorithms is problematic due to the specific properties of the echo chamber detection problem.

In echo chamber detection, interactions inside the echo chamber and between the echo chamber and outside should be considered (but not between nodes outside the echo chamber) which is not supported by the community detection problem as well as by classical theoretical computer science problems, such as maximum graph cut and maximum clique~\cite{garey_computers_2009}.
Echo chambers do not have to be cliques (i.e., where everybody interacts with one another).
As shown in Figure~\ref{fig:comparison} and discussed in more detail in Section~\ref{sec:objfun}, maximum graph cut fails to capture the interactions inside the echo chamber, while community detection unnecessarily captures the interactions between nodes outside the echo chamber.
To the best of our knowledge, we are the first to formalize the objective function of echo chamber detection, from theoretical principles.

\begin{figure*}
\centering
\begin{tabular}{ccccc}
\begin{tikzpicture}[scale=0.8]
\tikzstyle{inA}=[circle, fill=white, draw, inner sep=0pt, minimum size=15pt]
\tikzstyle{outA}=[circle, fill=white, draw, inner sep=0pt, minimum size=15pt]
\node[label=below:$\vphantom{holder}$, circle, fill=white, draw=none, inner sep=0pt, minimum size=50pt] at (0.5,0.4) {};
\node[label=below:$\vphantom{holder}$, circle, fill=white, draw=none, inner sep=0pt, minimum size=55pt] at (3,0.5) {};
\node[inA][label=center:$1$](in1) at (0.5,1) {};
\node[inA][label=center:$2$](in2) at (0,0) {};
\node[inA][label=center:$3$](in3) at (1,0) {};
\node[outA][label=center:$4$](out4) at (2.5,1) {};
\node[outA][label=center:$5$](out5) at (3.5,1) {};
\node[outA][label=center:$6$](out6) at (2.5,0) {};
\node[outA][label=center:$7$](out7) at (3.5,0) {};
\tikzset{EdgeStyle/.style={black,dotted}}
\Edge[lw=1.5pt](in1)(in3)
\Edge[lw=1.5pt](in2)(in3)
\Edge[lw=1.5pt](in1)(out4)
\Edge[lw=1.5pt](in3)(out4)
\Edge[lw=1.5pt](in3)(out6)
\Edge[lw=1.5pt](out4)(out5)
\Edge[lw=1.5pt](out5)(out6)
\Edge[lw=1.5pt](out5)(out7)
\end{tikzpicture} & \hspace{0.1in} & \begin{tikzpicture}[scale=0.8]
\tikzstyle{inA}=[circle, fill=white, draw, inner sep=0pt, minimum size=15pt]
\tikzstyle{outA}=[circle, fill=white, draw, inner sep=0pt, minimum size=15pt]
\node[label=below:$A$, circle, fill=gray!20, draw=none, inner sep=0pt, minimum size=50pt] at (0.5,0.4) {};
\node[label=below:$V \setminus A$, circle, fill=gray!20, draw=none, inner sep=0pt, minimum size=55pt] at (3,0.5) {};
\node[inA][label=center:$1$](in1) at (0.5,1) {};
\node[inA][label=center:$2$](in2) at (0,0) {};
\node[inA][label=center:$3$](in3) at (1,0) {};
\node[outA][label=center:$4$](out4) at (2.5,1) {};
\node[outA][label=center:$5$](out5) at (3.5,1) {};
\node[outA][label=center:$6$](out6) at (2.5,0) {};
\node[outA][label=center:$7$](out7) at (3.5,0) {};
\tikzset{EdgeStyle/.style={black,dotted}}
\Edge[lw=1.5pt](out4)(out5)
\Edge[lw=1.5pt](out5)(out6)
\Edge[lw=1.5pt](out5)(out7)
\tikzset{EdgeStyle/.style={blue}}
\Edge[lw=1.5pt](in1)(in3)
\Edge[lw=1.5pt](in2)(in3)
\tikzset{EdgeStyle/.style={red}}
\Edge[lw=1.5pt](in1)(out4)
\Edge[lw=1.5pt](in3)(out4)
\Edge[lw=1.5pt](in3)(out6)
\end{tikzpicture} & \hspace{0.1in} & \begin{tikzpicture}[scale=0.8]
\tikzstyle{inA}=[circle, fill=white, draw, inner sep=0pt, minimum size=15pt]
\tikzstyle{outA}=[circle, fill=white, draw, inner sep=0pt, minimum size=15pt]
\node[label=below:$A$, circle, fill=gray!20, draw=none, inner sep=0pt, minimum size=50pt] at (0.5,0.4) {};
\node[label=below:$V \setminus A$, circle, fill=gray!20, draw=none, inner sep=0pt, minimum size=55pt] at (3,0.5) {};
\node[inA][label=center:$1$](in1) at (0.5,1) {};
\node[inA][label=center:$2$](in2) at (0,0) {};
\node[inA][label=center:$3$](in3) at (1,0) {};
\node[outA][label=center:$4$](out4) at (2.5,1) {};
\node[outA][label=center:$5$](out5) at (3.5,1) {};
\node[outA][label=center:$6$](out6) at (2.5,0) {};
\node[outA][label=center:$7$](out7) at (3.5,0) {};
\tikzset{EdgeStyle/.style={black,dotted}}
\Edge[lw=1.5pt](in1)(in3)
\Edge[lw=1.5pt](in2)(in3)
\Edge[lw=1.5pt](out4)(out5)
\Edge[lw=1.5pt](out5)(out6)
\Edge[lw=1.5pt](out5)(out7)
\tikzset{EdgeStyle/.style={blue}}
\Edge[lw=1.5pt](in1)(out4)
\Edge[lw=1.5pt](in3)(out4)
\Edge[lw=1.5pt](in3)(out6)
\end{tikzpicture} \\
(a) Original graph &  & (b) Echo chamber detection &  & (c) Maximum graph cut
\end{tabular}
\vphantom{holder} \\
\begin{tabular}{ccc}
\begin{tikzpicture}[scale=0.8]
\tikzstyle{inA}=[circle, fill=white, draw, inner sep=0pt, minimum size=15pt]
\tikzstyle{outA}=[circle, fill=white, draw, inner sep=0pt, minimum size=15pt]
\node[label=below:$A$, circle, fill=gray!20, draw=none, inner sep=0pt, minimum size=50pt] at (0.5,0.4) {};
\node[label=below:$V \setminus A$, circle, fill=gray!20, draw=none, inner sep=0pt, minimum size=55pt] at (3,0.5) {};
\node[inA][label=center:$1$](in1) at (0.5,1) {};
\node[inA][label=center:$2$](in2) at (0,0) {};
\node[inA][label=center:$3$](in3) at (1,0) {};
\node[outA][label=center:$4$](out4) at (2.5,1) {};
\node[outA][label=center:$5$](out5) at (3.5,1) {};
\node[outA][label=center:$6$](out6) at (2.5,0) {};
\node[outA][label=center:$7$](out7) at (3.5,0) {};
\tikzset{EdgeStyle/.style={blue}}
\Edge[lw=1.5pt](in1)(in3)
\Edge[lw=1.5pt](in2)(in3)
\Edge[lw=1.5pt](out4)(out5)
\Edge[lw=1.5pt](out5)(out6)
\Edge[lw=1.5pt](out5)(out7)
\tikzset{EdgeStyle/.style={red}}
\Edge[lw=1.5pt](in1)(out4)
\Edge[lw=1.5pt](in3)(out4)
\Edge[lw=1.5pt](in3)(out6)
\end{tikzpicture} & \hspace{0.1in} & \begin{tikzpicture}[scale=0.8]
\tikzstyle{inA}=[circle, fill=white, draw, inner sep=0pt, minimum size=15pt]
\tikzstyle{outA}=[circle, fill=white, draw, inner sep=0pt, minimum size=15pt]
\node[label=below:$A$, circle, fill=gray!20, draw=none, inner sep=0pt, minimum size=50pt] at (0.5,0.4) {};
\node[label=below:$V \setminus A$, circle, fill=gray!20, draw=none, inner sep=0pt, minimum size=55pt] at (3,0.5) {};
\node[inA][label=center:$1$](in1) at (0.5,1) {};
\node[inA][label=center:$2$](in2) at (0,0) {};
\node[inA][label=center:$3$](in3) at (1,0) {};
\node[outA][label=center:$4$](out4) at (2.5,1) {};
\node[outA][label=center:$5$](out5) at (3.5,1) {};
\node[outA][label=center:$6$](out6) at (2.5,0) {};
\node[outA][label=center:$7$](out7) at (3.5,0) {};
\tikzset{EdgeStyle/.style={black,dotted}}
\Edge[lw=1.5pt](in1)(out4)
\Edge[lw=1.5pt](in3)(out4)
\Edge[lw=1.5pt](in3)(out6)
\Edge[lw=1.5pt](out4)(out5)
\Edge[lw=1.5pt](out5)(out6)
\Edge[lw=1.5pt](out5)(out7)
\tikzset{EdgeStyle/.style={blue}}
\Edge[lw=1.5pt](in1)(in3)
\Edge[lw=1.5pt](in2)(in3)
\end{tikzpicture} \\
(d) Community detection &  & (e) Maximum clique
\end{tabular}
\caption{(a) Original graph with edges shown as black dashed.
(b,c,d,e) Participation of edges in the objective function of echo chamber detection in eq.\eqref{eq:echochamberset}, maximum graph cut in eq.\eqref{eq:maxcut}, community detection in eq.\eqref{eq:commdetect} and maximum clique.
Red solid edges are those which are added to the objective function.
Blue solid edges are those which are subtracted from the objective function.
Black dashed edges are not used in the objective function.
In our example, $A=\{1,2,3\}$ is a candidate set of nodes (e.g., echo chamber) used in the different objective functions.
The set of nodes outside $A$ is $V \setminus A=\{4,5,6,7\}$.
In echo chamber detection, the edges inside $A$, and the edges between $A$ and $V \setminus A$ are considered, but the edges inside $V \setminus A$ are not considered.
This is not supported by the other objective functions.}
\label{fig:comparison}
\end{figure*}
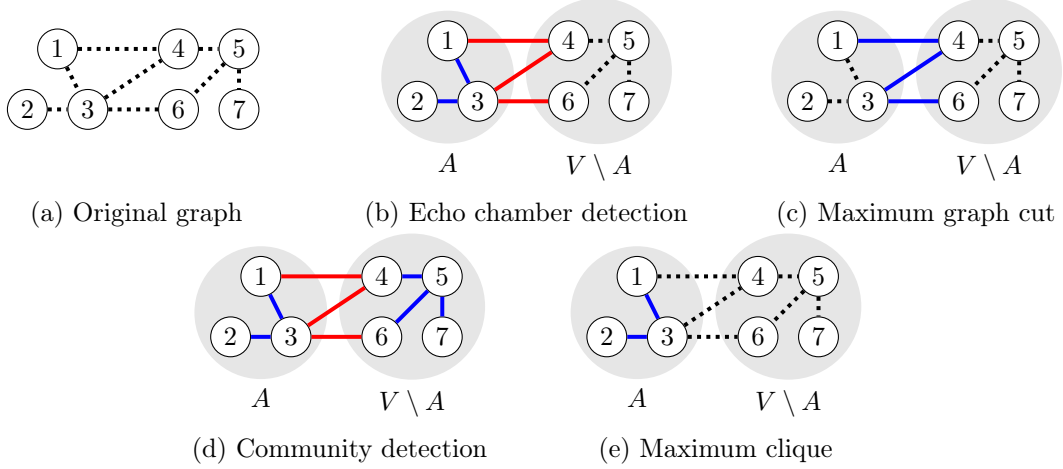

Our contribution is four-fold.
First, we derive a proper objective function for the problem of echo chamber detection, by using the theory of Fourier transforms of set functions.
Our objective function encourages more agreements and fewer disagreements inside the echo chamber, while encouraging more disagreements and fewer agreements to outside the echo chamber.
Second, we propose scalable semidefinite relaxation for our proposed problem, solved via an interior point method and sparse linear algebra.
Third, our algorithm recovers the ground truth echo chamber better than competing methods on small synthetic experiments.
Our algorithm produces echo chambers with better network properties than competing methods on large real-world datasets.
Fourth, as an independent validation of our objective function, we show that our algorithm finds echo chambers with more agreements with suspended users than competing methods on a small real-world dataset.

\section{Preliminaries} \label{sec:prelim}

In this section, we introduce the main notations and concepts that are used throughout the paper.
We denote sets by uppercase letters (e.g., $A$).
We use lowercase bold letters for vectors (e.g., $\v$) and uppercase bold letters for matrices (e.g., $\M$), and subscripts of non-bold letters to denote respective entries of a vector or matrix (e.g., $v_i$ and $m_{ij}$).
For a matrix $\M$, its trace is denoted by $\tr(\M)$ and its determinant is denoted by $\det(\M)$.
A vector of all ones (or all zeros) of size $n$ is denoted as $\one_n$ (or $\zero_n$).
We denote by $\e_i$ a vector which has zero entries everywhere except for entry $i$ which contains $1$.
An identity matrix of size $n$ is denoted as $\I_n$.
$\M \succ \zero_n\zero_n^\top$ (or $\M \succeq \zero_n\zero_n^\top$) denotes matrix $\M$ is positive definite (or positive semidefinite).
We denote the inner product of two matrices $\H$ and $\M$ as $\innerprod{\H}{\M} = \sum_{ij} h_{ij} m_{ij}$.
For a vector $\v$, $\Diag(\v)$ is the matrix containing $\v$ on its diagonal, and zero everywhere else.
For a matrix $\M$, $\diag(\M)$ is the vector containing the diagonal entries of $\M$.

The social network is represented as a graph of $n$ nodes.
The node set of the graph is $V = \{1, \dots, n\}$.
Edges are represented by a signed adjacency weight matrix $\W \in \R^{n \times n}$.
We consider undirected graphs and thus, $\W$ is symmetric and $\diag(\W) = \zero_n$.
A positive edge weight $w_{ij}>0$ represents reinforcing interaction between node $i$ and $j$.
A negative edge weight $w_{ij}<0$ represents an antagonistic interaction between node $i$ and $j$.
A zero weight $w_{ij}=0$ represents no interaction (i.e., no edge) between node $i$ and $j$.
In this paper, we focus on detecting one echo chamber.
We denote by $r$ the size of the echo chamber to be identified.

Our analysis relies on the theory of Fourier transforms of set functions.
The power set $2^V$ represents the set of all subsets of $V$ of zero elements, one element, two elements, and so on and so forth, up to $n$ elements.
That is, $2^V = \{ A \mid A \subseteq V \} = \{ \emptyset, \{1\}, \dots, \{n\}, \{1,2\}, \dots, \{n-1,n\}, \dots, \{1,\dots,n\} \}$.
For any arbitrary set function $f : 2^V \to \R$, its Fourier coefficients are defined as~\cite{Stobbe2012}:
\begin{equation} \label{eq:fourier}
\fh(B) = 2^{-n} \sum_{A \in 2^V} f(A) (-1)^{|A \cap B|}
\end{equation}
for all $B \in 2^V$.
Furthermore, we can reconstruct the original set function $f$ from its Fourier coefficients:
\begin{equation} \label{eq:arbitraryset}
f(A) = \sum_{B \in 2^V} \fh(B) (-1)^{|A \cap B|},
\end{equation}

Next, we provide a couple of examples of Fourier coefficients from prior work.
The maximum graph cut objective function has the form~\cite{garey_computers_2009}:
\begin{equation} \label{eq:maxcut}
g(A) = \sum_{i \in A, j \notin A} w_{ij}
\end{equation}
and has Fourier coefficients~\cite{Stobbe2012}:
\begin{equation*}
\gh(B) = \begin{cases}
\frac{1}{2} \sum_{i<j} w_{ij} & \text{if } B=\emptyset, \\
-\frac{1}{2} w_{ij} & \text{if } B=\{i,j\}, \\
0 & \text{otherwise}.
\end{cases}
\end{equation*}
The community detection objective function has the form~\cite{abbe_community_2018}:
\begin{equation} \label{eq:commdetect}
g(A) = \sum_{i,j \in A,i<j} w_{ij} + \sum_{i,j \notin A,i<j} w_{ij} - \sum_{i \in A, j \notin A} w_{ij}
\end{equation}
and has Fourier coefficients:
\begin{equation*}
\gh(B) = \begin{cases}
w_{ij} & \text{if } B=\{i,j\}, \\
0 & \text{otherwise}.
\end{cases}
\end{equation*}

\section{Main Results} \label{sec:main}

In this section, we derive an objective function for echo chamber detection, by using the theory of Fourier transforms of set functions.
We then devise scalable semidefinite relaxation, solved via an interior point method and sparse linear algebra.

\subsection{Objective Function via Fourier Transform of Set Functions} \label{sec:objfun}

In echo chamber detection, the goal is to find a set of nodes $A$ such that there are many positive edges (and few negative edges) between nodes in $A$, while having very few positive edges (and many negative edges) between nodes in $A$ and outside $A$.
This definition encompasses the intuition of echo chambers.
More formally, we define the set function $f : 2^V \to \R$ to be:
\begin{equation} \label{eq:echochamberset}
f(A) = \sum_{i,j \in A, i < j} w_{ij} - \sum_{i \in A, j \notin A} w_{ij}
\end{equation}

Assume $A$ is a candidate (echo chamber) set of nodes.
Note that when compared to eq.\eqref{eq:echochamberset}, the maximum graph cut objective function in eq.\eqref{eq:maxcut} disregards the term $\sum_{i,j \in A, i < j} w_{ij}$ and thus, maximum graph cut fails to capture the interactions inside the echo chamber $A$.
Also note that when compared to eq.\eqref{eq:echochamberset}, the community detection objective function in eq.\eqref{eq:commdetect} has an extra unnecessary term $\sum_{i,j \notin A,i<j} w_{ij}$ which captures the interactions between nodes outside the echo chamber $A$.
Furthermore, both eq.\eqref{eq:maxcut} and eq.\eqref{eq:commdetect} are symmetric, in the sense that $g(A)=g(V \setminus A)$, meaning that they return the same value for an echo chamber ($A$) or for all the nodes outside the echo chamber ($V \setminus A$).
The discussion above makes eq.\eqref{eq:maxcut} and eq.\eqref{eq:commdetect} not suitable for echo chamber detection.

In order to devise a scalable algorithm, our initial goal is to figure out whether there exists an equivalent (and efficient) representation of the \emph{set} function $f : 2^V \to \R$ as a \emph{multivariate} function $\fb : \{-1,+1\}^n \to \R$.
In order to do this, we denote by $\x \equiv \x(A) \in \{-1,+1\}^n$ a binary $n$-dimensional vector that encodes set $A$ as follows:
\begin{equation} \label{eq:x}
x_i(A) = 
\begin{cases}
-1 & \text{if } i \in A, \\
+1 & \text{if } i \notin A.
\end{cases}
\end{equation}
Our initial goal is then to find a multivariate function $\fb : \{-1,+1\}^n \to \R$ such that $f(A)=\fb(\x(A))$ for all $A \in 2^V$.

We highlight that such multivariate function might not necessarily have an efficient representation.
As shown in Lemma~\ref{lem:settomultivariate}, a multivariate function $f : \{-1,+1\}^n \to \R$ could potentially depend on high-order monomials.
Fortunately, we show in Lemma~\ref{lem:fourier} that we only need terms for up to size $2$ in the set function (i.e., $|B|\leq 2$) or equivalently, up to quadratic terms in the multivariate function.

First, we show that any set function has an equivalent multivariate function representation.

\begin{lemma} \label{lem:settomultivariate}
Given an arbitrary set function $f : 2^V \to \R$, the equivalent multivariate function $\fb : \{-1,+1\}^n \to \R$ such that $f(A)=\fb(\x(A))$ for all $A \in 2^V$, is given by:
\[
\fb(\x) = \sum_{B \in 2^V} \fh(B) \prod_{i \in B} x_i
\]
\end{lemma}
\begin{proof}
From eq.\eqref{eq:arbitraryset}, we have:
$ %\[
f(A) = \sum_{B \in 2^V} \fh(B) (-1)^{|A \cap B|}
$. %\]
We can observe that in order obtain the desired result, we need to show that for all $A,B \in 2^V$, we have:
$ %\[
\prod_{i \in B} x_i(A) = (-1)^{|A \cap B|}
$. %\]
By eq.\eqref{eq:x} and since $x_i(A) \in \{-1,+1\}$, we have:
\begin{align*}
\prod_{i \in B} x_i(A) & = \left( \prod_{i \in B, i \in A} x_i(A) \right) \left( \prod_{i \in B, i \notin A} x_i(A) \right) %\\
% & 
= \left( \prod_{i \in A \cap B} (-1) \right) \left( \prod_{i \in B \setminus A} 1 \right) %\\
% & 
= (-1)^{|A \cap B|}
\end{align*}
which proves our claim.
\end{proof}

Next, we derive the Fourier coefficients of our our echo-chamber detection set function eq.\eqref{eq:echochamberset}.
We highlight that terms of size greater than $2$ in the set function (i.e., $|B|> 2$) have zero Fourier coefficient.
Thus, from Lemma~\ref{lem:settomultivariate}, we only need up to quadratic terms in the multivariate function.

\begin{lemma} \label{lem:fourier}
For our echo-chamber detection set function in eq.\eqref{eq:echochamberset}, the Fourier coefficients are:
\begin{equation*}
\fh(B) = \begin{cases}
-\frac{1}{4} \sum_{i<j} w_{ij} & \text{if } B=\emptyset, \\
-\frac{1}{4} \sum_{j\neq i} w_{ij} & \text{if } B=\{i\}, \\
\frac{3}{4} w_{ij} & \text{if } B=\{i,j\}, \\
0 & \text{otherwise}.
\end{cases}
\end{equation*}
\end{lemma}
(The proof is included in Appendix~\ref{app:fourier}.)
\begin{proof}[Proof sketch]
We start with the Fourier coefficient definition in eq.\eqref{eq:fourier} and analyze the size of the set $A \cap B$ for the different cases of $B$ and for all $A$.
This reasoning leads to summands of the Fourier coefficient for which $|A \cap B|$ is even and thus $(-1)^{|A \cap B|}=1$, and summands for which $|A \cap B|$ is odd and thus $(-1)^{|A \cap B|}=-1$.
We then simplify the expressions to arrive to our claimed results.
\end{proof}

Given our previous results, we show multivariate function for echo chamber detection.
This will later help us take advantage of the machinery of convex relaxations to construct a scalable algorithm.

\begin{theorem} \label{thm:fourier}
For our echo-chamber detection set function in eq.\eqref{eq:echochamberset}, the equivalent multivariate function $\fb : \{-1,+1\}^n \to \R$ such that $f(A)=\fb(\x(A))$ for all $A \in 2^V$, is given by:
\begin{align*}
\fb(\x) = & -\frac{1}{4} \left(\sum_{i<j} w_{ij}\right) - \frac{1}{4} \sum_i \left(\sum_{j\neq i} w_{ij}\right) x_i %\\
% &
+ \frac{3}{4} \sum_{i,j} w_{ij} x_i x_j
\end{align*}
\end{theorem}
\begin{proof}
By Lemmas~\ref{lem:settomultivariate} and \ref{lem:fourier}.
\end{proof}

\subsection{Scalable Semidefinite Relaxation} \label{sec:sdp}

Armed with an equivalent and efficient representation of the set function as a multivariate function, here we devise a scalable algorithm using a convex relaxation, interior-point methods and sparse linear algebra.

Echo chamber detection aims to find a set of size $r$ what maximizes the objective function in eq.\eqref{eq:echochamberset}.
Thus, the original (set) combinatorial optimization problem can be expressed as:
\begin{align*}
\maximize & f(A) \\
\subjectto & A \in 2^V, %\\
% &
\quad |A|=r.
\end{align*}
Given our Lemma~\ref{lem:settomultivariate}, we can express the above as the following (multivariate) combinatorial problem:
\begin{align} \label{eq:combiopt}
\maximize & \fb(\x) \nonumber \\
\subjectto & \x \in \{-1,+1\}^n, %\nonumber \\
% & 
\quad \one_{n+1}^\top \x = n-2r.
\end{align}
Note that since given eq.\eqref{eq:x}, having $|A|=r$ is equivalent to having $r$ minus ones and $n-r$ ones in vector $\x$, and thus the equivalent constraint is $\one_{n+1}^\top \x = n-2r$.

We now define the Fourier matrix $\Fh \in \R^{(n+1)\times(n+1)}$ as follows:
\[
\Fh = \begin{bmatrix}
0 & \fh(\{1,2\}) & \cdots & \fh(\{1,n\}) & \fh(\{1\}) \\
\fh(\{1,2\}) & 0 & \cdots & \fh(\{2,n\}) & \fh(\{2\}) \\
\vdots & \vdots & \ddots & \vdots & \vdots \\
\fh(\{1,n\}) & \fh(\{2,n\}) &\cdots & 0 & \fh(\{n\}) \\
\fh(\{1\}) & \fh(\{2\}) &\cdots & \fh(\{n\}) & \fh(\emptyset)
\end{bmatrix},
\]
The above definition allows to write our result in Theorem~\ref{thm:fourier} as follows:
\begin{equation} \label{eq:objective}
\fb(\x) = \begin{bmatrix}\x \\ 1\end{bmatrix}^\top \Fh \begin{bmatrix}\x \\ 1\end{bmatrix}
\end{equation}

Instead of working directly with the vector $\x$, we consider its lifted form $\X = \begin{bmatrix}\x \\ 1\end{bmatrix} \begin{bmatrix}\x \\ 1\end{bmatrix}^\top$, where $\X \in \R^{(n+1)\times(n+1)}$ is a positive semidefinite matrix of rank 1.
In this representation, the objective becomes linear since
$\begin{bmatrix}\x \\ 1\end{bmatrix}^\top \Fh \begin{bmatrix}\x \\ 1\end{bmatrix} = \tr\left(\begin{bmatrix}\x \\ 1\end{bmatrix}^\top \Fh \begin{bmatrix}\x \\ 1\end{bmatrix}\right) = \tr\left(\Fh \begin{bmatrix}\x \\ 1\end{bmatrix} \begin{bmatrix}\x \\ 1\end{bmatrix}^\top\right) = \innerprod{\Fh}{\X}$.
The combinatorial constraint $\x \in \{-1,+1\}^n$ is equivalent to the constraint $\diag(\X) = \one_{n+1}$ since $x_i^2=1$ for all $i$.
Let $\H = \begin{bmatrix}\zero_n\zero_n^\top & \one_n \\ \one_n^\top & 0\end{bmatrix}$.
The constraint $\one_{n+1}^\top \x = n-2r$ is equivalent to $\innerprod{\H}{\X} = 2(n-2r)$ since $\X$ can also be written as $\X = \begin{bmatrix}\x\x^\top & \x \\ \x^\top & 1 \end{bmatrix}$.

For $\X \in \R^{(n+1)\times(n+1)}$, given our reasoning above and dropping the rank-1 constraint, we can relax the optimization problem eq.\eqref{eq:combiopt} as the following semidefinite program:
\begin{align} \label{eq:sdp}
\maximize & \innerprod{\Fh}{\X} \nonumber \\
\subjectto & \X \succeq \zero_{n+1}\zero_{n+1}^\top, %\nonumber \\
% &
\quad \diag(\X) = \one_{n+1}, %\nonumber \\
% &
\quad \innerprod{\H}{\X} = 2(n-2r),
\end{align}

We now devise a scalable solver for the above semidefinite relaxation.
We follow an interior point method~\cite{Boyd06}, which replaces the inequality constraints (i.e., $\X \succeq \zero_{n+1}\zero_{n+1}^\top$ in our problem) with a logarithmic barrier function (i.e., $\log \det(\X)$ in our case).
That is, for a logarithmic barrier factor $t>0$, we have:
\begin{align} \label{eq:interiorpoint}
\maximize & \innerprod{\Fh}{\X} + \frac{1}{t} \log \det(\X) \nonumber \\
\subjectto & \diag(\X) = \one_{n+1}, %\nonumber \\
% & 
\quad \innerprod{\H}{\X} = 2(n-2r),
\end{align}
The logarithmic barrier $\log \det(\X)$ ensures that $\X$ remains strictly within the interior of the semidefinite cone, i.e., $\X \succ \zero_n \zero_n^\top$~\cite{Boyd06}, thereby aiding convergence and preventing numerical degeneracy.

We then follow a dual gradient ascent approach that updates the dual variables (i.e., $\bnu \in \R^{n+1}$ associated with the constraint $\diag(\X) = \one_{n+1}$, and $\lambda \in \R$ associated with the constraint $\innerprod{\H}{\X} = 2(n-2r)$) and can recover the primal variable $\X$ at any iteration.
Algorithm~\ref{alg:interiorpoint} describes our method.
(Full derivation is included in Appendix~\ref{app:interiorpoint}.)

\begin{minipage}{0.48\textwidth}
\begin{algorithm}[H]%[t]
\caption{Interior Point Method for Solving Eq.\eqref{eq:sdp}}
\label{alg:interiorpoint}
\begin{algorithmic}[1]
  \STATE \textbf{Input:} Fourier matrix $\Fh \in \R^{(n+1) \times (n+1)}$, echo chamber size $r$, logarithmic barrier factor $t>0$, number of iterations $L$, step size $\eta>0$.
  \STATE $\bnu \leftarrow \zero_{n+1}$
  \STATE $\lambda \leftarrow 0$
  \FOR{$l = 1,\dots,L$}
    \STATE $\M \leftarrow -t \, \Fh + \Diag(\bnu) + \lambda \H$
    \STATE $\X \leftarrow \inv{\M}$
    \STATE $\bnu \leftarrow \bnu + \eta \, (\diag(\X) - \one_{n+1})$
    \STATE $\lambda \leftarrow \lambda + \eta \, (\innerprod{\H}{\X} - 2(n-2r))$
  \ENDFOR
  \STATE $\begin{bmatrix}\x \\ 1\end{bmatrix} \leftarrow$ the maximum eigenvector of $\X$
  \STATE \textbf{Output:} $\x \in \R^n$.
\end{algorithmic}
\end{algorithm}
\end{minipage}
\hfill
\begin{minipage}{0.48\textwidth}
\begin{algorithm}[H]%[t]
\caption{Scalable Interior Point Method for Solving Eq.\eqref{eq:sdp}}
\label{alg:interiorpointsparse}
\begin{algorithmic}[1]
  \STATE \textbf{Input:} Sparse Fourier matrix $\Fh \in \R^{(n+1) \times (n+1)}$, echo chamber size $r$, logarithmic barrier factor $t>0$, number of iterations $L$, step size $\eta>0$.
  \STATE $\bnu \leftarrow \zero_{n+1}$
  \STATE $\lambda \leftarrow 0$
  \FOR{$l = 1,\dots,L$}
    \STATE $\M \leftarrow -t \, \Fh + \Diag(\bnu) + \lambda \H$
    \FOR{$i = 1,\dots,n+1$ (easily parallelizable)}
        \STATE $\y = \solve(\M,\e_i)$
        \STATE $d_i = \y^\top \M \y$
    \ENDFOR
    \STATE $\z = \solve\left(\M,\begin{bmatrix}\one_n \\ 0\end{bmatrix}\right)$
    \STATE $\bnu \leftarrow \bnu + \eta \, (\d - \one_{n+1})$
    \STATE $\lambda \leftarrow \lambda + \eta \, (2 \, \y^\top \M \z - 2(n-2r))$
  \ENDFOR
  \STATE $\begin{bmatrix}\x \\ 1\end{bmatrix} \leftarrow$ the minimum eigenvector of $\M$
  \STATE \textbf{Output:} $\x \in \R^n$.
\end{algorithmic}
\end{algorithm}
\end{minipage}

When applying Algorithm~\ref{alg:interiorpoint}, there is the need to: compute an inverse $\X$ of a large matrix $\M$, compute the diagonal of $\X$, compute the inner product $\innerprod{\H}{\X}$, and compute the maximum eigenvector of $\X$.
Matrix $\Fh$ is sparse in practice, while $\H$ is sparse by definition, which makes $\M$ sparse as well.
Note that even though $\M$ is sparse, its inverse $\X$ is usually dense, which we also observe in practice.
This presents an issue for datasets with a large number of nodes $n$.
The goal is then to devise an algorithm that does not store $\X$ at any point.
To do this, we can take advantage of iterative linear equation system solvers, such as the Gauss-Seidel, Jaccobi, Richardson, successive over relaxation, minimal residual, among other methods.

Assume a black-box iterative linear equation system solver.
That is, $\y = \solve(\M,\v)$ returns a solution for $\M \y = \v$ or equivalently, $\y = \inv{\M}\v$.
Algorithm~\ref{alg:interiorpointsparse} describes our method.
(Full derivation is included in Appendix~\ref{app:interiorpointsparse}.)
For a graph with $E$ edges, the sparse matrix $\Fh$ has $O(n+E)$ nonzero entries, while $\H$ has $O(n)$ nonzero entries, thus, $\M$ has $O(n+E)$ nonzero entries.
Since iterative solvers are based on matrix-vector multiplications, their complexity is $O(n+E)$ for a sparse $\M$.
Thus, Algorithm~\ref{alg:interiorpointsparse} has a computational complexity of $O(Ln(n+E))$ and a space complexity of $O(n+E)$.
Note that with parallelization, the computational complexity could be improved.

\section{Experiments} \label{sec:exp}

In this section, we show that our method recovers the ground truth echo chamber better than competing methods on small synthetic experiments.
We also show that our method produces echo chambers with better network properties than competing methods on large real-world datasets.
Finally, as an independent validation, we show that our method finds echo chambers with more agreements with suspended users than competing methods on a small real-world dataset.

For all of our experiments and methods, we set the echo chamber size to be $r=\lceil \sqrt{n} \rceil$, where $n$ is the number of nodes.
We use our method in Algorithm~\ref{alg:interiorpointsparse} with the minimal residual method as the sparse linear equation system solver.

Most of the current research applies existing community detection algorithms to detect echo chambers.
We chose several popular community detection algorithms as comparison methods.
Louvain~\cite{blondel2008fast}, Signed Louvain~\cite{xia2021fast}, Girvan-Newman~\cite{newman2006modularity}, and Leiden~\cite{traag2019louvain} are the algorithms based on modularity and community structure metrics.
Infomap~\cite{rosvall2009map} and WalkTrap~\cite{pons2005computing} are the algorithms based on walk metrics.
They have been extensively used in community detection and echo chamber detection tasks~\cite{cossard2020falling, del2017mapping, alatawi2021survey, cinelli2021echo}.
Signed Louvain is a Louvain algorithm's~\cite{blondel2008fast} adaptation for signed graphs.
For these comparison methods, we first choose the smallest community with at least $r$ nodes, and then retain the $r$ nodes with highest number of neighbors.

\paragraph{Synthetic Data.}

\begin{figure}%[ht]
    \centering
    \includegraphics[width=0.5\columnwidth]{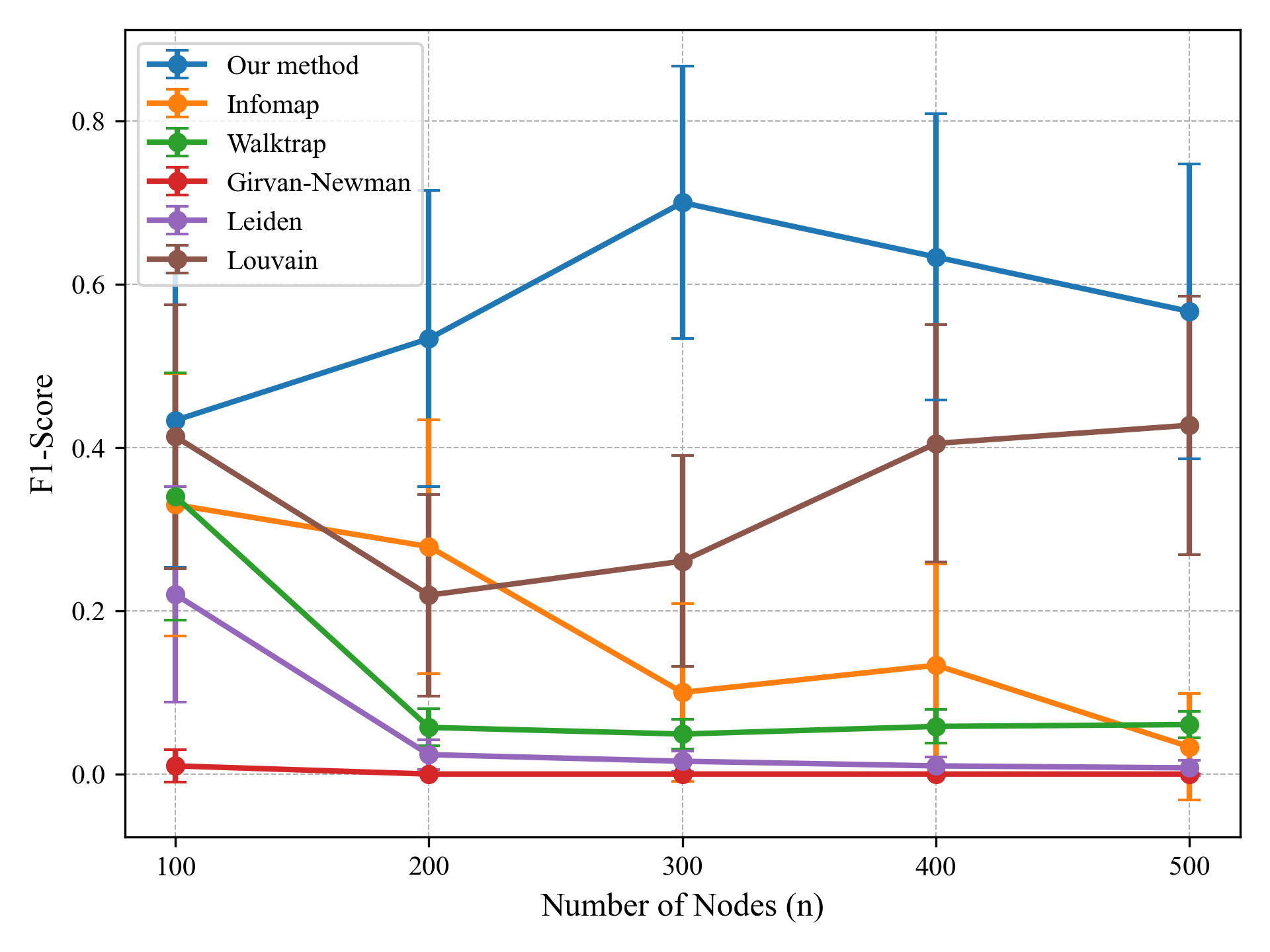}\includegraphics[width=0.5\columnwidth]{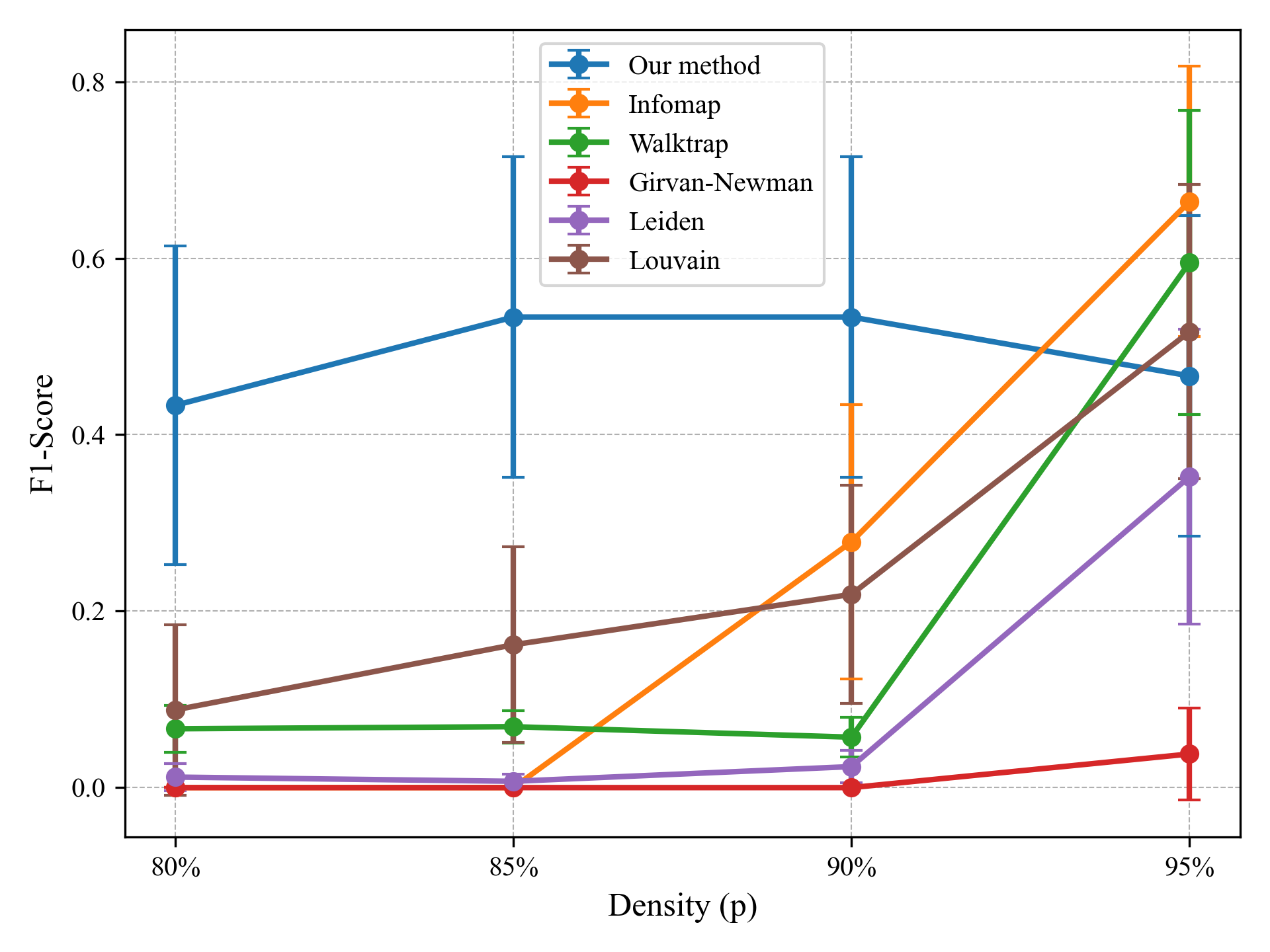}
    \vspace{-0.25in}
    \caption{\textbf{Left:} F1 score between the recovered echo chamber and the ground truth echo chamber, versus various number of nodes ($n$), for density ($p=90\%$).
    \textbf{Right:} F1 score between the recovered echo chamber and the ground truth echo chamber, versus various densities ($p$), for number of nodes ($n=200$).
    Error bars at 95\% confidence level over 30 repetitions.
    Our method outperforms others on recovering the ground truth echo chamber.
    Signed Louvain was not included due to their high computational demand.}
    \label{fig:f1score}
\end{figure}

\begin{table}%[t]
\centering
\caption{Number of nodes, positive edges and negative edges of real-world datasets used in our experiments.}
\label{tab:real_datasets}
\renewcommand{\arraystretch}{1.1}
\begin{tabular}{@{}l@{}ccc@{}}
\hline
\textbf{Dataset} & \textbf{Nodes} & \textbf{Positive} & \textbf{Negative} \\
 & $(n)$ & \textbf{edges} & \textbf{edges} \\
\hline
Russia-Ukraine & 246     & 624        & 142 \\
Facebook       & 4,039   & 176,468    & 0 \\
Abortion       & 7,242   & 3,105,862  & 132,758 \\
Obamacare      & 8,539   & 5,077,974  & 273,478 \\
Twitter        & 81,306  & 2,684,606  & 0 \\
Google         & 107,614 & 24,476,570 & 0 \\
\hline
\end{tabular}
\end{table}

\begin{table*}%[t]
\centering
\caption{Edge statistics (positive/negative edges inside the echo chamber, and positive/negative edges between the echo chamber and outside the echo chamber) and Cheeger constant, which measures the amount of connectivity inside the echo chamber.
For edge statistics, we also include percentages relative to all edges.
Bold values represent the best outcome.
Our method (Ours) consistently produces the most positive edges and best connectivity inside the echo chamber compared to Infomap (IM), Leiden (Le), Louvain (Lo), Walktrap (WT), Girvan-Newman (GN) and signed Louvain (SLo).
SLo was only included for the smallest dataset (Russia-Ukraine), and GN was only included the two smallest datasets (Russia-Ukraine and Facebook), due their high computational demand.
WT was not included for the largest dataset (Google) due to its high memory demand.}
\label{tab:real_results}
\begin{small}
\renewcommand{\arraystretch}{1.2}
\begin{tabular}{@{}l@{}ccccc@{}c@{}}
\hline
\textbf{Dataset} & \textbf{Method} & \textbf{Pos. edges} & \textbf{Neg. edges} & \textbf{Pos. edges} & \textbf{Neg. edges} & \textbf{Cheeger} \\
 &  & \textbf{inside} & \textbf{inside} & \textbf{between} & \textbf{between} & \textbf{constant} \\

\hline
Russia-Ukraine  & Ours & \textbf{38 (5.0\%)} & 0 (0.0\%) & 120 (15.7\%) & \textbf{17 (2.2\%)} & \textbf{1.54} \\
$n=$246  & IM  & 23 (3.0\%) & 0 (0.0\%) & 57 (7.4\%) & 11 (1.4\%) & 0.63 \\
  & Le   & 26 (3.4\%) & 0 (0.0\%) & 70 (9.1\%) & 4 (0.5\%) & 0.53 \\
  & Lo   & 23 (3.0\%) & 0 (0.0\%) & 61 (8.0\%) & 3 (0.4\%) & 0.68 \\
  & WT  & 31 (4.0\%) & 0 (0.0\%) & 134 (17.5\%) & 15 (2.0\%) & 0.79 \\
  & GN & 21 (2.7\%) & 0 (0.0\%) & 197 (25.7\%) & 6 (0.8\%) & 4.6e-16 \\
  & SLo  & 15 (2.0\%) & 0 (0.0\%) & \textbf{8 (1.0\%)}  &  1 (0.1\%) & 0.99 \\

\hline
Facebook  & Ours & \textbf{2008 (1.1\%)} & 0 (0.0\%)  & 9047 (5.1\%) & 0 (0.0\%) & \textbf{61.00} \\
$n=$4,039  & IM  & 1654 (0.9\%) & 0 (0.0\%)  & 9085 (5.1\%) & 0 (0.0\%) & 1.28 \\
  & Le   & 767 (0.4\%) & 0 (0.0\%) & \textbf{2176 (1.2\%)} & 0 (0.0\%) & 1.23 \\
  & Lo   & 742 (0.42\%) & 0 (0.0\%) & 2232 (1.26\%) & 0 (0.0\%) & 0.9 \\
  & WT  & 1241 (0.7\%) & 0 (0.0\%) & 3686 (2.1\%) & 0 (0.0\%) & 5.66 \\
  & GN & 935 (0.5\%) & 0 (0.0\%) & 14004 (7.9\%) & 0 (0.0\%) & 5e-15 \\

\hline
Abortion  & Ours & \textbf{3529 (0.11\%)} & 0 (0.0\%)  & 109329 (3.37\%) & 170 (0.01\%) & \textbf{287.26} \\
$n=$7,242  & IM  & 3422 (0.11\%) & 0 (0.0\%) & 113948 (3.52\%) & 297 (0.01\%) & 169.95 \\
  & Le   & 2557 (0.08\%) & 56 (0.002\%) & 41305 (1.28\%) & \textbf{5483 (0.17\%)} & 4.1e-15 \\
  & Lo   & 2675 (0.08\%) & 0 (0.0\%) & \textbf{40225 (1.24\%)} & 5101 (0.16\%) & 42.4 \\
  & WT  & 3422 (0.11\%) & 0 (0.0\%) & 113948 (3.52\%) & 297 (0.01\%) & 169.95 \\

\hline
Obamacare  & Ours & 4045 (0.08\%) & 0 (0.0\%)  & \textbf{56754 (1.06\%)} & \textbf{10483 (0.20\%)} & 253.27 \\
$n=$8,539  & IM  & 4079 (0.08\%) & 0 (0.0\%) & 177317 (3.31\%) & 416 (0.01\%) & 271.02 \\
  & Le  & 4069 (0.08\%) & 0 (0.0\%) & 176669 (3.30\%) & 406 (0.01\%) & 271.00 \\
  & Lo   & 3412 (0.06\%) & 0 (0.0\%) & 128331 (2.40\%) & 1454 (0.03\%) & 68.4 \\
  & WT  & \textbf{4089 (0.08\%)} & 0 (0.0\%) & 175424 (3.28\%) & 383 (0.01\%) & \textbf{289.40} \\

\hline
Twitter  & Ours & \textbf{17864 (0.67\%)} & 0 (0.0\%) & \textbf{65237 (2.43\%)} & 0 (0.0\%) & \textbf{41.28} \\
$n=$81,306  & IM & 10018 (0.37\%) & 0 (0.0\%) & 121624 (4.53\%) & 0 (0.0\%) & 3.00 \\
  & Le   & 6167 (0.23\%) & 0 (0.0\%) & 77701 (2.89\%) & 0 (0.0\%) & 3.50 \\
  & Lo   & 11025 (0.41\%) & 0 (0.0\%) & 105505 (3.93\%) & 0 (0.0\%) & 10.7 \\
  & WT  & 5739 (0.21\%) & 0 (0.0\%) & 114174 (4.25\%) & 0 (0.0\%) & 3.30 \\

\hline
Google  & Ours & \textbf{47941 (0.20\%)} & 0 (0.0\%) & 1606500 (6.56\%) & 0 (0.0\%) & \textbf{208.70} \\
$n=$107,614  & IM & 28312 (0.12\%) & 0 (0.0\%) & \textbf{1491175 (6.09\%)} & 0 (0.0\%) & 65.85 \\
  & Le   & 29088 (0.12\%) & 0 (0.0\%) & 1510019 (6.17\%) & 0 (0.0\%) & 70.60 \\
  & Lo   & 28888 (0.12\%) & 0 (0.0\%) & 1509087 (6.16\%) & 0 (0.0\%) & 65.8 \\

\hline
\end{tabular}
\end{small}
\end{table*}

Here, we perform experiments on small synthetic datasets.
We create graphs of $n$ nodes, following an approach similar to that for Erd\"{o}s–R\'{e}nyi graphs, that produces graphs with properties resembling those of echo chambers.
We use a global parameter $p \in (0.3, 1)$ that controls the edge density.
(Full details are provided in Appendix~\ref{app:expdetails}.)
By using eq.\eqref{eq:x} we equivalently defined the ground truth vector $\xt = \x(\At) \in \{-1,+1\}^n$.
For all methods, we compute the F1 score between the recovered vector $\x \in \R^n$ and the ground truth vector $\xt$.
That is, we convert $\x$ to a sign vector.
Note that $\xt$ is already a sign vector.
We then compute precision and recall and therefore the F1 score.
For our method in Algorithm~\ref{alg:interiorpointsparse}, we used logarithmic barrier factor $t=0.08$, $L=100$ iterations and step size $\eta=0.1$.

We run 30 repetitions of the above procedure, and report the mean and standard error bars.
As observed in %Figures~\ref{fig:f1score_n} and \ref{fig:f1score_density}
Figures~\ref{fig:f1score}, our method recovers
the ground truth echo chamber better than competing methods on small synthetic experiments.

\paragraph{Real-World Data.}

In what follows, we perform experiments on real-world datasets.
We consider several datasets of various sizes, but favor mostly large datasets.
Table~\ref{tab:real_datasets} shows the statistics regarding number of nodes, positive edges and negative edges.
For all methods, we evaluate the recovered echo chamber with various meaningful metrics.
We compute the number of positive and negative edges between nodes in the echo chamber.
We also compute the number of positive negative edges between nodes in the echo chamber and nodes outside the echo chamber.
To measure the amount of overall/global connectivity inside the echo chamber, we use the Cheeger constant of the subgraph formed by the nodes in the echo chamber.
While computing this quantity is computationally intractable, one can approximate it by computing the second minimum eigenvalue of the Laplacian of the subgraph~\cite{cheeger1969lower}.
For our method in Algorithm~\ref{alg:interiorpointsparse}, we used logarithmic barrier factor $t=0.0001$, $L=10$ iterations and step size $\eta=0.1$.

As observed in Table~\ref{tab:real_results}, our method produces echo chambers with most positive edges and best connectivity than those of competing methods on large real-world datasets.

\paragraph{Independent Validation.}

\begin{table}%[!h]
    \centering
    \caption{Analysis of interactions with suspended users through an unpaired t-test in the Russia-Ukraine dataset.
    $\Delta$SR is the difference between the mean of $SuspendedReplies$ of echo chamber members and the mean of $SuspendedReplies$ of non-echo chamber members.
    $\Delta$SA follows the same procedure for $SuspendedAgrees$.
    We also report their corresponding p-values.
    Our method (Ours) shows the largest and most statistically significant positive difference as compared to Infomap (IM), Leiden (Le), Louvain (Lo), Walktrap (WT), Girvan-Newman (GN) and signed Louvain (SLo).}
    \label{table:SuspendResults}
    \begin{tabular}{@{}ccccc@{}}
        \hline
        \textbf{Method} & \textbf{$\Delta$SR} & \textbf{p-value}   & \textbf{$\Delta$SA} & \textbf{p-value}  \\
    \hline
        Ours & \textbf{2.7054} & \textbf{0.0012} &  \textbf{2.6484}  & \textbf{0.0009}  \\
        IM & 0.7668   & 0.2797 &  0.5092 & 0.4158 \\
        Le & 1.2348  & 0.0591  & 1.3114  & 0.0381 \\
        Lo & 1.6359 & 0.0435 & 1.6457 & 0.0351 \\
        WT & 2.2375  & 0.007  & 2.1804   & 0.0054 \\
        GN & 1.8364  & 0.0323 & 1.8462 & 0.0241 \\
        SLo & -0.7707  & 0.0913  &-0.6272 & 0.1463 \\
        \hline
    \end{tabular}
\end{table}

As discussed in Section~\ref{sec:intro}, echo chambers are extensively associated with misinformation, aggressive content and extremist ideas.
In social networks, the users who perform such activities tend to be suspended by the administrators.
Therefore, it is reasonable to assume that echo chamber members have many interactions and agreements with the suspended users compared to the other users in the network.
Fortunately, the Russia-Ukraine dataset contains ground-truth information of which users were suspended.
We held this information and was not provided to any of the tested algortithms, including ours.
To assess whether an echo chamber detected by any method would likely be so in the real world, we use two measures, the number of replies a given user has with a suspended user ($SuspendedReplies$) and the number of agree replies a given user has with a suspended user ($SuspendedAgrees$).
Since we have two sets of users (echo chamber and non-echo chamber members), we perform an unpaired t-test.
In Table~\ref{table:SuspendResults}, we report the difference between the mean of $SuspendedReplies$ of echo chamber members and the mean of $SuspendedReplies$ of non-echo chamber members.
We follow the same procedure for $SuspendedAgrees$.
We also report their corresponding p-values.

As shown in Table~\ref{table:SuspendResults}, our method shows the largest positive $SuspendReplies$ difference.
The small p-value indicates that this difference is statistically significant.
This means that the echo chamber members detected by our method tend to interact with the suspended users more than the echo chamber members detected by the other algorithms.
Similarly, our method shows the largest positive $SuspendedAgrees$ difference with a small p-value.
This shows that the echo chamber members detected by our method tend to agree with the suspended users' opinions more than the echo chamber members detected by the other algorithms. 
The above results suggests that, compared to the baselines, the echo chambers detected by our method might more likely be actual echo chambers in the real world.

\section{Concluding Remarks}

Our contributions open several questions for future work.
While we could use our method recursively for identifying one echo chamber at a time, it would be interesting to have a direct generalization that identifies several echo chambers.
While we focused on pairwise interactions between two nodes which led to consider edges and graphs, it would be interesting to analyze the case where several nodes interact together which would lead to hyperedges and hypergraphs.

\bibliographystyle{plain}
\bibliography{references}

\clearpage
\appendix

\section{Detailed Proofs for Section \ref{sec:objfun}} \label{app:proofs}

In this section, we provide detailed proofs for the lemmas in the main text.

\subsection{Proof of Lemma~\ref{lem:fourier}} \label{app:fourier}

\begin{proof}
For brevity, we will shorten our notation and write $A$ instead of $A \in 2^V$ in the sums and sets.
Note that in all the proofs, we have $i \neq j$ in general, while in some cases we also have $i<j$.
We analyze four different cases.

\paragraph{Case 1.}

First, we consider the case $B=\emptyset$.
Note that in this case, we have $|A \cap B|=0$ for all $A$.
Therefore, by eq.\eqref{eq:fourier} we have:
\begingroup
\allowdisplaybreaks
\begin{align*}
\fh{(B)} & = 2^{-n} \sum_{A} f(A) (-1)^{|A \cap B|} \\
 & = 2^{-n} \sum_{A} f(A) \\
 & = 2^{-n} \sum_{A} \left( \sum_{i,j \in A, i < j} w_{ij} - \sum_{i \in A, j \notin A} w_{ij} \right) \\
 & = 2^{-n} \left( \sum_{i<j} \sum_{A \mid i,j \in A} w_{ij} - \sum_{i \neq j} \sum_{A \mid i \in A, j \notin A} w_{ij} \right) \\
 & = 2^{-n} \left( \sum_{i<j} w_{ij} \, |\{ A \mid i,j \in A \}| - \sum_{i \neq j} w_{ij} \, |\{ A \mid i \in A, j \notin A \}| \right) \\
 & = 2^{-n} \left( \sum_{i<j} w_{ij} \, |\{ A \mid i,j \in A \}| - 2 \sum_{i<j} w_{ij} \, |\{ A \mid i \in A, j \notin A \}| \right) \\
 & = 2^{-n} \left( \sum_{i<j} w_{ij} \, 2^{n-2} - 2 \sum_{i<j} w_{ij} \, 2^{n-2} \right) \\
 & = -\frac{1}{4} \sum_{i<j} w_{ij}
\end{align*}
\endgroup
where the second to the last line follows since $\left|\{ A \mid i,j \in A \}\right| = 2^{n-2}$ since we fix elements $i$ and $j$ in $A$ and we can choose any subset of $n-2$ other elements.
Similarly, $\left|\{ A \mid i \in A, j \notin A \}\right| = 2^{n-2}$ since we fix element $i$ in $A$ and $j$ not in $A$ and we can choose any subset of $n-2$ other elements.

\paragraph{Case 2.}

Second, we consider the case $B=\{k\}$.
Note that in this case, we have $|A \cap B|=0$ if $k \notin A$, and $|A \cap B|=1$ if $k \in A$.
Therefore, by eq.\eqref{eq:fourier} we have:
\begingroup
\allowdisplaybreaks
\begin{align*}
\fh{(B)} & = 2^{-n} \sum_{A} f(A) (-1)^{|A \cap B|} \\
 & = 2^{-n} \left( \sum_{A \mid k \notin A} f(A) - \sum_{A \mid k \in A} f(A) \right) \\
 & = 2^{-n} \left( \sum_{A \mid k \notin A} \left( \sum_{i,j \in A, i < j} w_{ij} - \sum_{i \in A, j \notin A} w_{ij} \right) - \sum_{A \mid k \in A} \left( \sum_{i,j \in A, i < j} w_{ij} - \sum_{i \in A, j \notin A} w_{ij} \right) \right) \\
 & = 2^{-n} \left( \sum_{i<j} \sum_{A \mid i,j \in A, k \notin A} w_{ij} - \sum_{i \neq j} \sum_{A \mid i \in A, j,k \notin A} w_{ij} - \sum_{i<j} \sum_{A \mid i,j,k \in A} w_{ij} + \sum_{i \neq j} \sum_{A \mid i,k \in A, j \notin A} w_{ij} \right) \\
 & = 2^{-n} \left( \sum_{i<j} \underbrace{\left( \sum_{A \mid i,j \in A, k \notin A} w_{ij} - \sum_{A \mid i,j,k \in A} w_{ij} \right)}_{F_1} + \sum_{i \neq j} \underbrace{\left( \sum_{A \mid i,k \in A, j \notin A} w_{ij} - \sum_{A \mid i \in A, j,k \notin A} w_{ij} \right)}_{F_2} \right)
\end{align*}
\endgroup
Next, we analyze the two inner terms in the above expression.
Regarding the first inner term:
\begin{align*}
F_1 & = \sum_{A \mid i,j \in A, k \notin A} w_{ij} - \sum_{A \mid i,j,k \in A, k \notin \{i,j\}} w_{ij} - \sum_{A \mid i,j,k \in A, k \in \{i,j\}} w_{ij} \\
 & = w_{ij} \, |\{A \mid i,j \in A, k \notin A\}| - w_{ij} \, |\{A \mid i,j,k \in A, k \notin \{i,j\}\}| - \sum_{A \mid i,k \in A} w_{ik} - \sum_{A \mid k,j \in A} w_{kj} \\
 & = - \sum_{A \mid i,k \in A} w_{ik} - \sum_{A \mid k,j \in A} w_{kj} \\
 & = -w_{ik} \, |\{A \mid i,k \in A\}| - w_{kj} \, |\{A \mid k,j \in A\}| \\
 & = -(w_{ik} + w_{kj}) \, 2^{n-2}
\end{align*}
where the third to the last line follows since $|\{A \mid i,j \in A, k \notin A\}| = |\{A \mid i,j,k \in A, k \notin \{i,j\}\}|$.
The last line follows since $|\{A \mid i,k \in A\}| = |\{A \mid k,j \in A\}| = 2^{n-2}$ as argued in Case 1.
Regarding the second inner term:
\begin{align*}
F_2 & = \sum_{A \mid i,k \in A, j \notin A} w_{ij} - \sum_{A \mid i \in A, j,k \notin A} w_{ij} \\
 & = w_{ij} \, |\{A \mid i,k \in A, j \notin A\}| - w_{ij} \, |\{A \mid i \in A, j,k \notin A\}| \\
 & = 0
\end{align*}
where the last line follows since $|\{A \mid i,k \in A, j \notin A\}| = |\{A \mid i \in A, j,k \notin A\}|$.
Thus, we have:
\begin{align*}
\fh(B) & = 2^{-n} (F_1 + F_2) \\
 & = 2^{-n} \sum_{i<j} \left( -(w_{ik} + w_{kj}) \, 2^{n-2} \right) \\
 & = -\frac{1}{4} \sum_{j \neq k} w_{kj}
\end{align*}

\paragraph{Case 3.}

Third, we consider the case $B=\{k,l\}$.
Note that in this case, we have $|A \cap B|=0$ if $k,l \notin A$, $|A \cap B|=1$ if $k \in A, l \notin A$, and $|A \cap B|=2$ if $k,l \in A$.
Therefore, by eq.\eqref{eq:fourier} we have:
\begingroup
\allowdisplaybreaks
\begin{align*}
\fh{(B)} & = 2^{-n} \sum_{A} f(A) (-1)^{|A \cap B|} \\
 & = 2^{-n} \left( \sum_{A \mid k,l \notin A \text{ or } k,l \in A} f(A) - \sum_{A \mid k \in A, l \notin A} f(A) \right) \\
 & = 2^{-n} \left( \sum_{A \mid k,l \notin A \text{ or } k,l \in A} \left( \sum_{i,j \in A, i < j} w_{ij} - \sum_{i \in A, j \notin A} w_{ij} \right) - \sum_{A \mid k \in A, l \notin A} \left( \sum_{i,j \in A, i < j} w_{ij} - \sum_{i \in A, j \notin A} w_{ij} \right) \right) \\
 & = 2^{-n} \left( \sum_{i<j} \sum_{\substack{A \mid i,j \in A, k,l \notin A \\ \text{or } i,j,k,l \in A}} w_{ij} - \sum_{i \neq j} \sum_{\substack{A \mid i \in A, j,k,l \notin A \\ \text{or } i,k,l \in A, j \notin A}} w_{ij} - \sum_{i<j} \sum_{A \mid i,j,k \in A, l \notin A} w_{ij} + \sum_{i \neq j} \sum_{A \mid i,k \in A, j,l \notin A} w_{ij} \right) \\
 & = 2^{-n} \left( \sum_{i<j} \underbrace{\left( \sum_{\substack{A \mid i,j \in A, k,l \notin A \\ \text{or } i,j,k,l \in A}} w_{ij} - \sum_{A \mid i,j,k \in A, l \notin A} w_{ij} \right)}_{F_1} + \sum_{i \neq j} \underbrace{\left( \sum_{A \mid i,k \in A, j,l \notin A} w_{ij} - \sum_{\substack{A \mid i \in A, j,k,l \notin A \\ \text{or } i,k,l \in A, j \notin A}} w_{ij} \right)}_{F_2} \right)
\end{align*}
\endgroup
Next, we analyze the two inner terms in the above expression.
Regarding the first inner term:
\begin{align*}
F_1 & = \sum_{A \mid i,j \in A, k,l \notin A} w_{ij} + \sum_{A \mid i,j,k,l \in A} w_{ij} - \sum_{A \mid i,j,k \in A, l \notin A} w_{ij} \\
 & = \sum_{A \mid i,j \in A, k,l \notin A} w_{ij} + \sum_{A \mid i,j,k,l \in A, (i,j) \neq (k,l)} w_{ij} + \sum_{A \mid i,j,k,l \in A, (i,j) = (k,l)} w_{ij} - \sum_{A \mid i,j,k \in A, l \notin A} w_{ij} \\
 & = \sum_{A \mid i,j \in A, k,l \notin A} w_{ij} + \sum_{A \mid i,j,k,l \in A, (i,j) \neq (k,l)} w_{ij} + \sum_{A \mid k,l \in A} w_{kl} - \sum_{A \mid i,j,k \in A, l \notin A} w_{ij} \\
 & = w_{ij} \, |\{A \mid i,j \in A, k,l \notin A\}| + w_{ij} \, |\{A \mid i,j,k,l \in A, (i,j) \neq (k,l)\}| \\
 & \qquad + \sum_{A \mid k,l \in A} w_{kl} - w_{ij} \, |\{A \mid i,j,k \in A, l \notin A\}| \\
 & = \sum_{A \mid k,l \in A} w_{kl} \\
 & = w_{kl} \, |\{A \mid k,l \in A\}| \\
 & = w_{kl} \, 2^{n-2}
\end{align*}
where the third to the last line follows since $|\{A \mid i,j \in A, k,l \notin A\}| + |\{A \mid i,j,k,l \in A, (i,j) \neq (k,l)\}| =$ $|\{A \mid i,j,k \in A, l \notin A\}|$.
The last line follows since $|\{A \mid k,l \in A\}| = 2^{n-2}$ as argued in Case 1.
Regarding the second inner term:
\begin{align*}
F_2 & = \sum_{A \mid i,k \in A, j,l \notin A} w_{ij} - \sum_{A \mid i \in A, j,k,l \notin A} w_{ij} - \sum_{A \mid i,k,l \in A, j \notin A} w_{ij} \\
 & = \sum_{A \mid i,k \in A, j,l \notin A, (i,j) \notin \{(k,l),(l,k)\}} w_{ij} + \sum_{A \mid i,k \in A, j,l \notin A, (i,j) \in \{(k,l),(l,k)\}} w_{ij} - \sum_{A \mid i \in A, j,k,l \notin A} w_{ij} - \sum_{A \mid i,k,l \in A, j \notin A} w_{ij} \\
 & = \sum_{A \mid i,k \in A, j,l \notin A, (i,j) \notin \{(k,l),(l,k)\}} w_{ij} + \sum_{A \mid k \in A, l \notin A} (w_{kl} + w_{lk}) - \sum_{A \mid i \in A, j,k,l \notin A} w_{ij} - \sum_{A \mid i,k,l \in A, j \notin A} w_{ij} \\
 & = w_{ij} \, |\{A \mid i,k \in A, j,l \notin A, (i,j) \notin \{(k,l),(l,k)\}\}| + \sum_{A \mid k \in A, l \notin A} (w_{kl} + w_{lk}) \\
 & \qquad - w_{ij} \, |\{A \mid i \in A, j,k,l \notin A\}| - w_{ij} \, |\{A \mid i,k,l \in A, j \notin A\}| \\
 & = \sum_{A \mid k \in A, l \notin A} (w_{kl} + w_{lk}) \\
 & = (w_{kl} + w_{lk}) \, |\{A \mid k \in A, l \notin A\}| \\
 & = (w_{kl} + w_{lk}) \, 2^{n-2}
\end{align*}
where the third to the last line follows from the fact that $|\{A \mid i,k \in A, j,l \notin A, (i,j) \notin \{(k,l),(l,k)\}\}| = $ $|\{A \mid i \in A, j,k,l \notin A\}| + |\{A \mid i,k,l \in A, j \notin A\}|$.
The last line follows since $|\{A \mid k \in A, l \notin A\}| = 2^{n-2}$ as argued in Case 1.
Thus, we have:
\begin{align*}
\fh(B) & = 2^{-n} (F_1 + F_2) \\
 & = 2^{-n} \sum_{i<j} \left( w_{kl} \, 2^{n-2} +  (w_{kl} + w_{lk}) \, 2^{n-2} \right) \\
 & = \frac{3}{4} w_{kl}
\end{align*}

\paragraph{Case 4.}

Finally, we consider the case $|B|>2$.
Note that we can write:
\begin{align*}
\fh{(B)} & = 2^{-n} \sum_{A} f(A) (-1)^{|A \cap B|} \\
 & = 2^{-n} \sum_{A} \left( \sum_{i,j \in A, i < j} w_{ij} - \sum_{i \in A, j \notin A} w_{ij} \right) (-1)^{|A \cap B|} \\
 & = 2^{-n} \left( \underbrace{\sum_{A} \sum_{i,j \in A, i < j} w_{ij} (-1)^{|A \cap B|}}_{F_1} - \underbrace{\sum_{A} \sum_{i \in A, j \notin A} w_{ij} (-1)^{|A \cap B|}}_{F_2} \right)
\end{align*}
We now argue that $F_1 = 0$.
Assume we take a set $A$, and elements $i,j \in A$.
Now choose an element $k \in B \setminus \{i,j\}$ and $k \notin A$.
Define $A' = A \cup \{k\}$.
Note that the parity of $|A \cap B|$ is different from the parity of $|A' \cap B|$ by construction, i.e., one quantity is odd and the other quantity is even.
That is, $(-1)^{|A \cap B|} = -(-1)^{|A' \cap B|}$.
Therefore, $F_1$ contains two summands $w_{ij} (-1)^{|A \cap B|} + w_{ij}  (-1)^{|A' \cap B|} = 0$.
By following this argument for all summands of $F_1$, we have that $F_1 = 0$.

We then argue that $F_2 = 0$.
Assume we take a set $A$, and elements $i \in A, j \notin A$.
Now choose an element $k \in B \setminus \{i,j\}$ and $k \in A$.
Define $A' = A \setminus \{k\}$.
Note that the parity of $|A \cap B|$ is different from the parity of $|A' \cap B|$ by construction, i.e., one quantity is odd and the other quantity is even.
That is, $(-1)^{|A \cap B|} = -(-1)^{|A' \cap B|}$.
Therefore, $F_2$ contains two summands $w_{ij} (-1)^{|A \cap B|} + w_{ij}  (-1)^{|A' \cap B|} = 0$.
By following this argument for all summands of $F_2$, we have that $F_2 = 0$.
Therefore $\fh(B) = 2^{-n} (F_1 + F_2) = 0$.
\end{proof}

\section{Additional Details for Section~\ref{sec:sdp}} \label{app:sdp}

In this section, we provide additional details for the algorithm in the main text.

\subsection{Derivation of our Interior Point Method in Algorithm~\ref{alg:interiorpointsparse}} \label{app:interiorpoint}

Since $t>0$, eq.\eqref{eq:interiorpoint} can be equivalently written as:
\begin{align*}
\minimize & -t \, \innerprod{\Fh}{\X} - \log \det(\X) \nonumber \\
\subjectto & \diag(\X) = \one_{n+1}, \nonumber \\
 & \innerprod{\H}{\X} = 2(n-2r),
\end{align*}
Let $\bnu \in \R^{n+1}$ be the dual variable associated with the constraint $\diag(\X) = \one_{n+1}$.
Let $\lambda \in \R$ be the dual variable associated with the constraint $\innerprod{\H}{\X} = 2(n-2r)$.
We now define the Lagrangian associated with this optimization problem:
\begin{align*}
L(\X,\bnu,\lambda) & = -t \, \innerprod{\Fh}{\X} - \log \det(\X) + \innerprod{\bnu}{\diag(\X) - \one_{n+1}} + \lambda (\innerprod{\H}{\X} - 2(n-2r)) \\
 & = -t \, \innerprod{\Fh}{\X} - \log \det(\X) + \innerprod{\Diag(\bnu)}{\X} - \one_{n+1}^\top \bnu + \lambda (\innerprod{\H}{\X} - 2(n-2r)) \\
 & = \innerprod{-t \, \Fh + \Diag(\bnu) + \lambda \H}{\X} - \log \det(\X) - \one_{n+1}^\top \bnu - 2(n-2r) \lambda
\end{align*}
Taking the gradient with respect to $\X$ and equating to zero, leads to:
\begin{align*}
\frac{\partial L}{\partial \X} = -t \, \Fh + \Diag(\bnu) + \lambda \H - \inv{\X} = \zero_{n+1} \zero_{n+1}^\top
\end{align*}
Solving the above for $\X$ leads to its optimal value with respect to the Lagrangian $L$, which is:
\begin{align*}
\Xopt = \inv{\left( -t \, \Fh + \Diag(\bnu) + \lambda \H \right)}
\end{align*}
We can now compute the objective function of the dual problem as follows:
\begin{align*}
 & g(\bnu,\lambda) \\
 & = \min_{\X} L(\X,\bnu,\lambda) \\
 & = L(\Xopt,\bnu,\lambda) \\
 & = \innerprod{-t \, \Fh + \Diag(\bnu) + \lambda \H}{\Xopt} - \log \det(\Xopt) - \one_{n+1}^\top \bnu - 2(n-2r) \lambda \\
 & = \innerprod{-t \, \Fh + \Diag(\bnu) + \lambda \H}{\inv{\left( -t \, \Fh + \Diag(\bnu) + \lambda \H \right)} } - \log \det\left(\inv{\left( -t \, \Fh + \Diag(\bnu) + \lambda \H \right)}\right) \\
 & \qquad - \one_{n+1}^\top \bnu - 2(n-2r) \lambda \\
 & = \tr(\I_{n+1}) + \log \det\left( -t \, \Fh + \Diag(\bnu) + \lambda \H \right) - \one_{n+1}^\top \bnu - 2(n-2r) \lambda \\
 & = (n+1) + \log \det\left( -t \, \Fh + \Diag(\bnu) + \lambda \H \right) - \one_{n+1}^\top \bnu - 2(n-2r) \lambda
\end{align*}

We now proceed with a gradient ascent approach in order to maximize the objective function of the dual problem~\cite{Boyd06}.
For this, we compute the gradient with respect to $\bnu$, which is:
\begin{align*}
\frac{\partial g}{\partial \nu_i} & = \innerprod{\e_i \e_i^\top}{\inv{\left( -t \, \Fh + \Diag(\bnu) + \lambda \H \right)}} - 1 \\
  & = \innerprod{\e_i \e_i^\top}{\Xopt} - 1 \\
  & = \Xopt_{ii} - 1
\end{align*}
Therefore, we have:
\begin{align*}
\frac{\partial g}{\partial \bnu} & = \diag(\Xopt) - \one_{n+1}
\end{align*}
Then, we compute the derivative with respect to $\lambda$, which is:
\begin{align*}
\frac{\partial g}{\partial \lambda} & = \innerprod{\H}{\inv{\left( -t \, \Fh + \Diag(\bnu) + \lambda \H \right)}} - 2(n-2r) \\
 & = \innerprod{\H}{\Xopt} -2(n-2r)
\end{align*}
For a step size $\eta>0$, the final dual gradient ascent approach, is given by the update rules: $\bnu \leftarrow \bnu + \eta \, \frac{\partial g}{\partial \bnu}$ and $\lambda \leftarrow \lambda + \eta \, \frac{\partial g}{\partial \lambda}$.

\subsection{Derivation of our Scalable Interior Point Method in Algorithm~\ref{alg:interiorpointsparse}} \label{app:interiorpointsparse}

Recall that the logarithmic barrier $\log \det(\X)$ ensures that $\X$ remains strictly within the interior of the semidefinite cone, i.e., $\X \succ \zero_n \zero_n^\top$~\cite{Boyd06}.
Since $\X = \inv{\M}$ this also implies that $\M \succ \zero_n \zero_n^\top$.

First, we reason about the diagonal of $\X$ since $\diag(\X)$ is involved in the gradient ascent update of $\bnu$.
Let $\d = \diag(\X)$, i.e., $d_i = x_{ii}$.
Note that $x_{ii}$ can be written as:
\begin{align*}
d_i & = x_{ii} \\
 & = \e_i^\top \X \e_i \\
 & = \e_i^\top \inv{\M} \e_i \\
 & = \e_i^\top \inv{\M} \M \inv{\M} \e_i \\
 & = \y^\top \M \y
\end{align*}
for $\y = \inv{\M} \e_i$.
Second, we reason about the inner product $\innerprod{\H}{\X}$ which is involved in the gradient ascent update of $\lambda$.
Note that $\H$ can be written as:
\begin{align*}
\H = \begin{bmatrix}\one_n \\ 0\end{bmatrix} \e_{n+1}^\top + \e_{n+1} \begin{bmatrix}\one_n \\ 0\end{bmatrix}^\top
\end{align*}
Given the above, we have:
\begin{align*}
\innerprod{\H}{\X} & = \tr(\H \X) \\
 & = \tr\left( \left( \begin{bmatrix}\one_n \\ 0\end{bmatrix} \e_{n+1}^\top + \e_{n+1} \begin{bmatrix}\one_n \\ 0\end{bmatrix}^\top \right) \inv{\M} \right) \\
 & = \tr\left( \e_{n+1}^\top \inv{\M} \begin{bmatrix}\one_n \\ 0\end{bmatrix} + \begin{bmatrix}\one_n \\ 0\end{bmatrix}^\top \inv{\M} \e_{n+1} \right) \\
 & = 2 \, \e_{n+1}^\top \inv{\M} \begin{bmatrix}\one_n \\ 0\end{bmatrix} \\
 & = 2 \, \e_{n+1}^\top \inv{\M} \M \inv{\M} \begin{bmatrix}\one_n \\ 0\end{bmatrix} \\
 & = 2 \, \y \M \z
\end{align*}
where $\y = \inv{\M} \e_{n+1}$ and $\z = \inv{\M} \begin{bmatrix}\one_n \\ 0\end{bmatrix}$.
Finally, since $\X = \inv{\M}$ and since $\X \succ \zero_n \zero_n^\top$, the maximum eigenvector of $\X$ is the minimum eigenvector of $\M$.

\iffalse
trace(X H) = trace( X ( v w^T + w v^T ) )
= trace( X v w^T + X w v^T )
= trace( w^T X v + v^T  X w )
= 2 w^T X v
= 2 w^T M^{ -1} v
= 2 w^T M^{ -1} M M^{ -1} v
\fi

\section{Additional Experimental Details} \label{app:expdetails}

\subsection{Synthetic Data: Graph Creation}

We create graphs of $n$ nodes, following an approach similar to that for Erd\"{o}s–R\'{e}nyi graphs, that produces graphs with properties resembling those of echo chambers.
We use a global parameter $p \in (0.3, 1)$ that controls the edge density.
First, we choose $r$ nodes uniformly at random from the $n$ nodes, to be the ground truth echo chamber.
Then, we create edges as follows.
Let $\At$ be the ground truth echo chamber with $|\At|=r$ nodes.
For nodes $i,j \in \At$, we create an edge (i.e., $w_{ij} \neq 0$) with probability $p$.
Further, if $w_{ij} \neq 0$, then $w_{ij}=1$ with probability $p$ and $w_{ij}=-1$ with probability $1-p$.
For nodes $i,j \notin \in \At$, we create an edge (i.e., $w_{ij} \neq 0$) with probability $1-p$.
Further, if $w_{ij} \neq 0$, then $w_{ij}=1$ or $w_{ij}=-1$ with the same probability (50\%).
For node $i \in \At$ and node $j \notin \At$, we create an edge (i.e., $w_{ij} \neq 0$) with probability $1.2-p$.
Further, if $w_{ij} \neq 0$, then $w_{ij}=1$ with probability $1-p$ and $w_{ij}=-1$ with probability $p$.

For instance, for $p=90\%$, the edge density inside the echo chamber is approximately $p=0.9$, and among those edges $p=90\%$ are positive, resembling the expected interactions inside an echo chamber.
The edge density outside the echo chamber is approximately $1-p=10\%$, and among those edges $50\%$ are positive, resembling what we expect on a regular social network.
The edge density between the echo chamber and nodes outside, is approximately $1.2-p=30\%$, and among those edges $1-p=10\%$ are positive, resembling the expected interactions between the echo chamber and nodes outside the echo chamber.

\end{document}